\newcommand{\mytitle}{A Framework for Approval-based \\ Budgeting Methods}

\newcommand{\shortcite}[1]{\cite{#1}}
\newcommand{\journal}[1]{}

\documentclass[letterpaper]{article}
\usepackage{times}  
\usepackage{helvet}  
\usepackage{courier}  
\usepackage{url}  
\usepackage{graphicx}  

\usepackage{amsmath,amsfonts,amsthm}
\usepackage{makecell}
\usepackage{xcolor}
\usepackage{bbm}
\usepackage{nicefrac}
\usepackage{xspace}
\usepackage{pifont}

\theoremstyle{definition}
\newtheorem{definition}{Definition}
\newtheorem{proposition}{Proposition}
\newtheorem{observation}{Observation}
\newtheorem{corollary}{Corollary}
\newtheorem{example}{Example}
\newtheorem{remark}{Remark}

\DeclareMathOperator*{\argmax}{arg\,max}

\newcommand{\mypara}[1]{\bigskip\noindent\textbf{#1}}

\newcommand{\np}{{{\mathrm{NP}}}}

\newcommand{\xmark}{\text{\sffamily x}}
\newcommand{\cmark}{\checkmark}

\begin{document}

\pagestyle{plain}

\title{\mytitle}

\author{
  Piotr Faliszewski \\
  \small{AGH University} \\
  \small{Krakow, Poland}  
  \and
  Nimrod Talmon \\
  \small{Ben-Gurion University} \\
  \small{Be'er Sheva, Israel}
}

\maketitle

\begin{abstract}
We define and study a general framework for approval-based budgeting methods and compare certain methods within this framework by their axiomatic and computational properties. Furthermore, we visualize their behavior on certain Euclidean distributions and analyze them experimentally.
\end{abstract}

\section{Introduction}

Participatory budgeting~\cite{cabannes2004participatory}, initiated by the Brazil workers' party~\cite{wainwright2003making}, is gaining increased attention, and is currently applied on many continents, including North America~\cite{Gilm12a} and Europe (e.g., Paris is organizing one of the largest citywide participatory budgeting processes\footnote{\url{https://budgetparticipatif.paris.fr/bp}}).
The general premise of participatory budgeting is to let residents of a municipality influence the way by which their common funds are being distributed, through a deliberative grassroots process.
Concretely, residents are participating in constructing the municipal budget, by acting as voters and specifying their preferences over a set of available items; then, an aggregation mechanism (i.e., a budgeting method) is applied to decide upon the exact set of items to be funded.

Even though more and more funds are decided through participatory budgets, not many budgeting methods have been proposed, and no mathematical frameworks to allow for a systematic comparison between budgeting methods are available, rendering their use somewhat ad-hoc.
Here we describe such a general framework for the approval-based setting, in which voters specify subsets of the available items which they approves of. Corresponding methods within our framework differentiate in two aspects:
(1) The way by which voter satisfaction from a given set of funded items is defined, modeled through a \emph{satisfaction function};
and
(2) The way these satisfaction functions are used.

We consider several concrete methods within our framework, including some which are used in practice and
some which generalize  known multiwinner voting rules.
To compare these methods, we 
(1) consider their computational complexity;
(2) define several axioms, relevant to budgeting methods, and study how well these axioms are satisfied by the methods at hand;
and
(3) report on three experiments:
  In the first two, we visualize the behavior of these methods on certain Euclidean preferences,
  by adapting the methodology of Elkind et al.~\shortcite{elkind2017multiwinner}, originally developed for multiwinner voting rules;
  in the third experiment we assess how well our budgeting methods deal with local and global items.

The main contributions of our work are:
(1) a general framework of approval-based budgeting methods and highlight several methods within this framework;
(2) several useful axiomatic properties which are relevant to budgeting methods at large;
(3) an adaptation of the methodology of Elkind et al.~\shortcite{elkind2017multiwinner} to participatory budgeting;
and
(4) an evaluation of certain methods within our framework according to their axiomatic, computational, and visual properties,  as well as their ability to deal with local and global items.

\subsection{Related Work}

%\mypara{Related Work}
%
Researchers have considered ordinal-based budgeting methods, in which voters rank budget items~\cite{goel2015knapsack,condorcetbudgeting};
utility-based budgeting methods, in which voters have numerical utilities over budget items~\cite{fluschnik2017fair,benade2017preference};
and, as we do, approval-based budgeting methods, in which each voter approves a set of items~\cite{goel2015knapsack,goel2016budget,benade2017preference,aziz2017proportionally}.
Specifically, Goel et al.~\shortcite{goel2015knapsack,goel2016budget} study
$k$-Approval, where each voter approves exactly $k$ items,
and Knapsack voting, where each voter approves items while respecting the budget limit.
As shown below, the aggregation method used by Goel et al.,
in which the winning budget is selected by greedily considering the items in decreasing number of approvals,
fits within our framework.

Benade et al.~\shortcite{benade2017preference} consider the implicit utilitarian model~\cite{caragiannis2017subset} and analyze the distortion achieved by eliciting the preferences of the voters by Knapsack voting and by Threshold voting
(in which voters are asked to approve those items which they evaluate above a given threshold);
their distortion-based aggregation methods do not fit within our framework.
Aziz et al.~\shortcite{aziz2017proportionally} generalize multiwinner proportionality axioms to the setting of participatory budgeting.
Their budgeting method, which is a generalization of Phragm\'{e}n's sequential method for committee elections,
does not fit within our framework.

As we are interested in axiomatic properties of budgeting methods,
we mention the work of Shapiro and Talmon~\shortcite{condorcetbudgeting},
which considers a generalization of the Condorcet principle,
the paper of Aziz et al.~\shortcite{aziz2017proportionally},
which considers axioms of representation,
and the papers of Fluschnik et al.~\shortcite{fluschnik2017fair} and Fain et al.~\shortcite{fain2018fair},
which also consider representation.
Our framework of approval-based budgeting methods is a generalization of the framework studied by Lackner and Skowron~\shortcite{ABCMW},
which regards approval-based multiwinner methods.
Our framework, as well as theirs, might be seen as the approval-based variant of the framework of committee scoring rules~\cite{elkind2017properties,faliszewski2016multiwinner,csrhierarchy}.

\section{Budgeting Methods}

%We describe our model of approval-based participatory budgeting,
%consider \emph{satisfaction functions},
%provide some examples,
%and discuss few approaches to use them.

\subsection{Budgeting Scenarios}

We consider the following model of participatory budgeting.
A \emph{budgeting scenario} $E$ is
a tuple $E = (A, V, c, \ell)$
where $A = \{a_1, \ldots, a_m\}$ is a set of items,
$c \colon A \to \mathbb{N}$ is a \emph{cost function}, so that the \emph{cost} of item $a \in A$ is $c(a)$,
$V = \{v_1, \ldots, v_n\}$ is a set of voters,
where each voter $v \in V$ specifies her approval set $A_v \subseteq A$, containing those items which she approves of,
and $\ell \in \mathbb{N}$ is a budget limit.

A \emph{budgeting method} $\mathcal{R}$ is a function which
takes a budgeting scenario $E = (A, V, c, \ell)$
and returns a \emph{budget} $B \subseteq A$, such that the total cost of the items of $B$ respects the budget limit;
i.e., slightly abusing notation, it must hold that $c(B) = \sum_{b \in B} c(b) \leq \ell$.
The winning budget (i.e., the set of funded items) for a budgeting scenario $E$, under a budgeting method $\mathcal{R}$, is denoted by $\mathcal{R}(E)$.
With respect to a winning budget $B$, an item is \emph{budgeted} (or \emph{funded}) if it is contained in $B$.
For simplicity, we ignore issues related to tie-breaking, which can be dealt with using standard techniques.

\subsection{Satisfaction Functions}

Budgeting methods in our framework operate by considering the satisfaction of the voters from possible budgets.
To measure the satisfaction of a voter from a possible winning budget, we use satisfaction functions, defined below.

\begin{definition}[Satisfaction function]
A \emph{satisfaction function} $f$ is a function $f : 2^A \times 2^A \to \mathbb{N}$,
where $A$ is a set of items.
Given a budgeting scenario with a set of items $A$,
a voter $v \in V$ with her approval set $A_v$,
and a budget $B \subseteq A$,
the value $f(A_v, B)$ is the \emph{satisfaction} of $v$ from the budget $B$.
\end{definition}

\begin{example}
Consider a budgeting scenario with $A = \{a, b, c\}$,
a voter $v$ with $A_v = \{a, b\}$,
and a budget $B = \{b, c\}$.
For a satisfaction function $f$,
the value $f(A_v, B) = f(\{a, b\}, \{b, c\})$ is the satisfaction of $v$ from the budget $B$.
\end{example}

Next we describe a few satisfaction functions.
%(Indeed, while other satisfaction functions shall be considered in future work,
%here we concentrate on these functions.)
For a budget $B$ and a voter $v$ with her approval set $A_v$,
let $B_v := A_v \cap B$.

\begin{enumerate}

\item 
$f(A_v, B) = |B_v|$: The satisfaction of a voter is the number of budgeted items she approves of.

\item
$f(A_v, B) = \sum_{a \in B_v} c(a) = c(B_v)$: A voter's satisfaction is the total cost of her approved items which are budgeted.

\item
$f(A_v, B) = \mathbbm{1}_{|B_v| > 0}$: A voter has satisfaction $0$ if none of her approved items is budgeted, and $1$ otherwise (i.e., if at least one of her approved items is budgeted).

\end{enumerate}

\journal{%
\begin{remark}
Similarly to the syntactic hierarchy of committee scoring rules~\cite{csrhierarchy},
we mention a syntactic hierarchy of satisfaction functions:
  At the top there are those which depend on $A_v$ and on $B$;
  then, those which depend on $A_v$ and on $B_v$ (e.g., all the above examples are such);
  last, those which depend on $A_v$ and on $|B_v|$ (e.g., the first two examples above)
  and those which depend on $A_v$ and on $\{c(a) | a \in B_v\}$ (e.g, the last two examples above).
\end{remark}
}

\subsection{Using Satisfaction Functions}

A given satisfaction function can be used in various ways.
We consider the following three approaches.

\begin{enumerate}

\item 
\textbf{Max rules:}
For a satisfaction function $f$, the rule $\mathcal{R}^m_f$ selects,
as a winning budget,
a budget which maximizes the sum of voters' satisfaction, according to $f$.
Formally, $\mathcal{R}^m_f$ selects $\argmax_{B \subseteq A} \sum_{v \in V} f(A_v, B)$.

\item
\textbf{Greedy rules:}
For a satisfaction function $f$,
the rule $\mathcal{R}^g_f$ proceeds in iterations,
maintaining a partial budget $B$,
where in each iteration it adds an item $a$ to $B$ which maximizes the value $\sum_{v \in V} f(A_v, B \cup \{a\})$.

\item
\textbf{Proportional greedy rules:}
The rule $\mathcal{R}^p_f$ is similar to $\mathcal{R}^g_f$,
except that in each iteration it adds an item $a$ to $B$ which maximizes:
\[
  \left(\textstyle\sum_{v \in V} f(A_v, B \cup \{a\}) - \sum_{v \in V} f(A_v, B)\right) /\ c(a).
\]

\end{enumerate}

\begin{remark}
The above examples of satisfaction functions and approaches result in $9$ budgeting methods
which we discuss throughout the paper.
Indeed,
studying other functions and approaches is an immediate future work direction.
We chose those functions and approaches as they are natural, generalize known multiwinner voting rules, and include some known budgeting methods:
  $\mathcal{R}^m_{\mathbbm{1}_{|B_v| > 0}}$ generalizes approval-based Chamberlin--Courant~\cite{chamberlincourant} and $\mathcal{R}^g_{\mathbbm{1}_{|B_v| > 0}}$ generalizes the greedy approximation of this rule~\cite{lu2011budgeted}.
Furthermore, $\mathcal{R}^g_{|B_v|}$ is similar to the popular $k$-Approval and Knapsack voting~\cite{goel2015knapsack} (the aggregation method is the same, albeit $k$-Approval and Knapsack restrict the voter approval sets; we do not consider such restrictions in our framework, as we are interested in the aggregation).
\end{remark}

Below, we study the computational and axiomatic properties of these nine budgeting methods,
and report on simulations.

\newcommand{\smallish}[1]{\small{#1}}

\begin{table*}[t]
\centering
\resizebox{\textwidth}{!}{\begin{tabular}{c | c c c c c c c c c} 
& $\mathcal{R}^m_{|B_v|}$ & $\mathcal{R}^g_{|B_v|}$ & $\mathcal{R}^p_{|B_v|}$ & 
$\mathcal{R}^m_{\mathbbm{1}_{|B_v| > 0}}$ & $\mathcal{R}^g_{\mathbbm{1}_{|B_v| > 0}}$ & $\mathcal{R}^p_{\mathbbm{1}_{|B_v| > 0}}$ & 
$\mathcal{R}^m_{c(B_v)}$ & $\mathcal{R}^g_{c(B_v)}$ & $\mathcal{R}^p_{c(B_v)}$
\\ \hline
\smallish{Complexity} &
P & P & P & NP-h & P & P & weak NP-h & P & P \\
\smallish{Budget Mono.} &
\cmark & \cmark  & \cmark & \cmark & \cmark & \cmark & \cmark & \cmark & \cmark \\
\smallish{Discount Mono.} &
\cmark & \cmark & \cmark & \cmark & \cmark & \cmark & \xmark & \xmark & \xmark \\
\smallish{Splitting Mono.} &
\cmark & \cmark & \cmark & \cmark & \cmark & \cmark & \cmark & \xmark & \cmark \\
\smallish{Merging Mono.} &
\xmark & \xmark & \xmark & \cmark & \cmark & \xmark & \cmark & \cmark & \cmark \\
\smallish{Limit Mono.} &
\xmark & \xmark & \xmark & \xmark & \xmark & \xmark & \xmark & \xmark & \xmark \\
\end{tabular}}
\caption{Computational and axiomatic properties of certain approval-based budgeting methods.}
\label{table:axioms}
\end{table*}

\section{Budgeting Algorithms}

We consider the computational complexity of identifying winning budgets.
First, it follows from the definitions of the Greedy rules and the Proportional greedy rules that computing their winners, given that the functions used can be computed efficiently, can be done in polynomial time; this holds as these rules are defined through efficient iterative processes. This is not the case for Max rules, which in general are NP-hard. (To be concrete, next we consider a specific Max rule; to be formally correct, we show NP-hardness for deciding whether a budget with at least a given total satisfaction, i.e., sum of satisfaction values, exists.)

\begin{observation}\label{observation:nphardone}
  Given a budgeting scenario and a bound~$s$, deciding whether a feasible budget $B$ for which $\sum_{v \in V} \mathbbm{1}_{|B_v| > 0} \geq s$ exists is NP-hard.
\end{observation}

\begin{proof}
Notice that $\mathcal{R}^m_{\mathbbm{1}_{|B_v| > 0}}$ generalizes the CC rule~\cite{chamberlincourant}, which is $\np$-hard~\cite{procaccia2008complexity}.
Also, when all items are of unit cost, the problem is equivalent to Max Cover~\cite{GJ79,sko-fal:j:max-cover}.
\end{proof}

Next we consider other Max rules.

\begin{proposition}\label{proposition:polytimeone}
  Identifying a winning budget under $\mathcal{R}^m_{|B_v|}$ can be done in polynomial time.
\end{proposition}

\begin{proof}
An efficient algorithm can be realized by noticing the similarity to the \textsc{Knapsack} problem~\cite{GJ79}
and recalling that \textsc{Knapsack} is solvable in polynomial-time whenever one of the dimensions is given in unary.
Specifically, the relevant information for an item $a$ is its number of approvals $S(a)$ and its cost\footnote{%
  In the Knapsack literature,
  these are usually referred to as the \emph{value} and the \emph{weight} of the item;
  we use the jargon of budgeting scenarios and not the Knapsack jargon.}
$c(a)$;
and, while the cost $c(a)$ of any item $a$ is given in binary, and thus can be superpolynomial in the input size,
the number $S(a)$ of approvals of $a$ can be at most $|V|$, and thus polynomial in the input size.
Thus a dynamic program,
that computes the values of $T(i, z)$,
$i \in [m]$, $z \in [n \cdot m]$,
which stand for the cost of the cheapest budget with total satisfaction of at least $z$ from the first $i$ items,
runs in polynomial time.
To be more concrete, it is useful to view the input election as a binary matrix with $n$ rows and $m$ columns, such that the $(i,j)$th cell is $1$ if the $i$th voter approves the $j$th item and $0$ otherwise.
Let $S$ be the number of ones in this representation of the input election as a binary matrix.
Then, we will populate a table $T$ with $m$ rows and $Z$ columns, such that the $(i,z)$th cell would hold the cost of the cheapest budget which achieves a total satisfaction of \emph{exactly} $z$ while considering only the first $i$ items, or $0$ if it is not possible, where by ``not possible'' we mean either that there is no subset of items which would give this exact total satisfaction or that we could not afford any such subset, due to the budget limit.

To populate the first row, $T(1,z)$, for each $z \in Z$ we set $T(1,z)$ to $c(a_1)$ if the total satisfaction of $a_1$ is $z$, and $0$ otherwise.
To populate the $i$th row, $T(i,z)$, for each $z \in Z$ we first compute the total satisfaction $S(j)$ of $a_j$; then, we go over the cells $T(j,z - S(j))$, $j \in [i - 1]$ and locate the cell with the minimum value $v$; if all these cells are $0$ then we set $T(i, z)$ to $0$ as well, while if not then we set $T(i, z)$ to $v + c(a_j)$; now, we go over the cells $T(j,z)$, $j \in [i - 1]$, and if there is a cell whose value is strictly positive but strictly less than $v + c(a_j)$ then we set $T(i,z)$ to this value (this corresponds to not funding the item $a_j$).
After populating the table we go over the columns, starting from $Z$ towards $1$ and locate the first column which is not completely filled with zeros; then, we take the cell with the minimum value from this column which is not zero as the cost of the winning budget (the budget itself can be computed by usual techniques of dynamic programming).
\end{proof}
    
\begin{proposition}\label{proposition:weakone}
  Identifying winning budgets under $\mathcal{R}^m_{c(B_v)}$ can be done in pseudopolynomial time. Further, given a budgeting scenario and a bound $s$, deciding whether there is a feasible budget $B$ with $\sum_{v \in V} f(A_v, B) \geq s$, for $f(A_v, B) = \sum_{a \in B_v} c(a)$, is weakly NP-hard.
\end{proposition}
  
\begin{proof}
Weak NP-hardness follows by a straightforward reduction from the \textsc{Subset Sum} problem~\cite{GJ79}:
  Given a Subset Sum instance
  with integers $X = \{x_1, \ldots, x_n\}$,
  where the existence of a subset $X' \subseteq X$ with $\sum_{x \in X'} x = Z$ is to be decided,
  we construct the following budgeting scenario:
    For each integer $x_i$, we create a voter $v_i$ approving an item $a_i$ of cost $c(a_i) = x_i$,
  and set the limit to be $\ell = Z$.
  For a yes-instance of Subset Sum, a winning budget shall be of total cost $Z$, thus weak NP-hardness follows.
  
The reduction above also hints on the pseudopolynomial-time algorithm for $\mathcal{R}^m_{c(B_v)}$:
  Apply dynamic programming similar to that for Subset Sum, by iterating over the items and remembering the maximum total satisfaction that can be achieved for each amount of money.
\end{proof}

\subsection{Coping with Intractability}

We describe a simple Integer Linear Program (ILP) for Max rules over satisfaction functions which can be defined using ILPs;
note that all satisfaction functions considered here (i.e., $|B_v|$, $\mathbbm{1}_{|B_v| > 0}$, and $c(B_v)$)
can be defined using ILPs.
This is useful due to the availability of efficient ILP solvers.
Indeed, the simulations reported below were performed using the Gurobi ILP solver~\cite{gurobi} on such ILP formulations.

\begin{observation}\label{observation:ilp}
  Let $f$ be a satisfaction function which can be formulated as an ILP. Then, identifying a winning budget under $\mathcal{R}^m_f$ can be done using an ILP.
\end{observation}

\begin{proof}
Introduce a binary variable $x_a$ for each $a \in A$, which is $1$ if and only if the item $a$ is to be budgeted. Add a budget constraint $\sum_{a \in A} c(a)\cdot x_a \leq \ell$ and set the objective to be $\max \sum_{v \in V} f(A_v, B)$.
\end{proof}

Next is a general, parameterized complexity result.

\begin{observation}\label{observation:fptone}
  For all Max rules, the problem of deciding whether a winning budget of total cost at least a given value is fixed-parameter tractable for the number $m$ of items,
  but there are some Max rules for which this problem is hard even when $n = 1$, where $n$ is the number of voters.
\end{observation}

\begin{proof}
Fixed-parameter tractability with respect to the number $m$ of items follows by considering all feasible budgets.
Para-NP-hardness with respect to the number $n$ of voters holds, e.g., for the rule $\mathcal{R}^m_{c(B_v)}$,
by observing that it can encode instances of Subset Sum~\cite{GJ79} even with just one voter:
  Create one voter which approves all the items, where each item has cost equal to a number from the Subset Sum instance. %whose cost equals the value of the number.
Then, to decide whether there is a winning budget of total satisfaction at least the value asked for in the Subset Sum instance
  corresponds to deciding the Subset Sum instance.
\end{proof}

For $\mathcal{R}^m_{c(B_v)}$, which is weakly NP-hard (see Proposition~\ref{proposition:weakone})
even when there is only one voter (see the proof of Observation~\ref{observation:fptone}),
we have a pseudopolynomial time algorithm (see the proof of Proposition~\ref{proposition:weakone}) and an FPTAS,
as we show next.

\begin{observation}
  There is an FPTAS for $\mathcal{R}^m_{c(B_v)}$.
\end{observation}

\begin{proof}
Notice that given an instance of $\mathcal{R}^m_{c(B_v)}$,
one can reduce it to an instance of %$0/1$-
Knapsack:
  For each item, create a Knapsack element of \emph{weight} equal to the cost of the item and of \emph{value} equal to the cost of the item times the number of voters approving this item.  
Then, the result follows from the existence of an FPTAS for Knapsack~\cite{vazirani2013approximation}.
\end{proof}	

As for $\mathcal{R}^m_{\mathbbm{1}_{|B_v| > 0}}$,
which is NP-hard (see Observation~\ref{observation:nphardone}),
and for which there is no approximation algorithm with better than $1 - 1 / e$ approximation ratio (as it generalizes the multiwinner voting rule CC for approval elections, which itself is equivalent to the
Max Cover problem~\cite{fei:j:set-cover,sko-fal:j:max-cover}),
next we show that it is fixed-parameter tractable for the number $n$ of the voters.

\begin{proposition}
  $\mathcal{R}^m_{\mathbbm{1}_{|B_v| > 0}}$ is fixed-parameter tractable for $n$.
\end{proposition}

\begin{proof}
Guess the partition of the voters with the intended meaning that each group of voters in the partition is represented by the same item.
For each such group, guess the number of voters which would be satisfied.
Then, go over all items and pick as representative the cheapest item that makes exactly this number of voters in this group satisfied.
\end{proof}

\section{Budgeting Axioms}

In this section we suggest several axiomatic properties which are relevant to budgeting methods.
In particular, we focus on axioms which relate to the costs of the items.
For each axiom, after providing the definition we check which of the rules in our framework satisfy it.

Our first axiom models the very natural expectation that if within a
budgeting scenario we can afford to budget more items, then we
should. Formally, we express it as follows.

\begin{definition}[Budget Monotonicity]
  A budgeting method $\mathcal{R}$ satisfies Budget Monotonicity
% if a superset of items is always preferable to the set of items.
% Formally, 
if for each budgeting scenario $E = (A, V, c, \ell)$ and each pair of
feasible budgets $B$ and $B'$ such that $B \subset B'$ it holds that
if $B$ is winning then also $B'$ is winning.
\end{definition}

Perhaps surprisingly, the next example, admittedly somewhat artificial, demonstrates that not all rules in our framework satisfy Budget Monotonicity. 

\begin{example}
Let $f(A_v, B) = \min(\{c(a) : a \in B_v\})$;
that is, the satisfaction of a voter equals the cost of her cheapest approved item which is budgeted.
Then, $\mathcal{R}^m_f$ does not satisfy Budget monotonicity
(to see it, consider two items $a$, $b$ of cost $1$, $2$ respectively, one voter approving both $a$ and $b$, and a pair of budgets $B = \{a\}$ and $B' = \{a, b\}$).
\end{example}

Nevertheless, notice that all three satisfaction functions we consider
here (namely $f = |B_v|$, $f = \mathbbm{1}_{|B_v| > 0}$, and
$f = c(B_v)$) are super-set monotone; that is, for each of them, we
have that $f(B') \geq f(B)$ for each $B \subset B'$.  Thus, since one
can verify that Max rules are Budget Monotone for super-set monotone
satisfaction functions, we have the following.

\begin{corollary}\label{corollary:bmone}
  $\mathcal{R}^m_{|B_v|}$,
  $\mathcal{R}^m_{c(B_v)}$,
  and 
  $\mathcal{R}^m_{\mathbbm{1}_{|B_v| > 0}}$,
  satisfy Budget Monotonicity.
\end{corollary}

For Greedy rules and Proportional greedy rules, it is never the case
that two feasible budgets $B$, $B'$ with $B \subset B'$ are both
winning, so we have the following.

\begin{corollary}\label{corollary:bmtwo}
  $\mathcal{R}^g_{|B_v|}$,
  $\mathcal{R}^g_{c(B_v)}$,
  $\mathcal{R}^g_{\mathbbm{1}_{|B_v| > 0}}$,
  $\mathcal{R}^p_{|B_v|}$,
  $\mathcal{R}^p_{c(B_v)}$,
  and 
  $\mathcal{R}^p_{\mathbbm{1}_{|B_v| > 0}}$,
  satisfy Budget Monotonicity.
\end{corollary}

In the next axiom we consider the response of our rules to increasing
the available limit. Specifically, we require that if we increase the
limit, then all budgeted items remain budgeted, provided that no new
item becomes affordable (this last condidtion is quite natural; if
there is an item that all the voters approve, which is above
the budget limit before its extension but is within the limit after
the extension, then it is quite natural that this item might be budgeted
and might remove many previously budgeted ones). This axiom is
analogous to the committee monotonicity axiom from the world of
multiwinner elections~\cite{elkind2017properties}.

\begin{definition}[Limit Monotonicity]
%
  % A budgeting method $\mathcal{R}$ satisfies Limit Monotonicity if increasing the budget limit such that no new items becomes available does not cause any budgeted item to cease being such.
  We say that a budgeting method $\mathcal{R}$ satisfies Limit Monotonicity if for each pair of budgeting scenarios $E = (A, V, c, \ell)$, $E' = (A, V, c, \ell + 1)$ with no item which costs exactly $\ell + 1$,
for each $a \in A$, it holds that $a \in \mathcal{R}(E) \implies a \in \mathcal{R}(E')$.
\end{definition}

\begin{proposition}\label{proposition:limitmonotonicityone}
  Neither of
  $\mathcal{R}^m_{|B_v|}$,
  $\mathcal{R}^m_{c(B_v)}$,
  and 
  $\mathcal{R}^m_{\mathbbm{1}_{|B_v| > 0}}$
  satisfies Limit Monotonicity.
\end{proposition}

\begin{proof}
Below we consider each of the three rules separately.

\mypara{$\boldsymbol{\mathcal{R}^m_{\mathbbm{1}_{|B_v| > 0}}}$:}
Consider a budgeting scenario $E$ with items $a$, $b$, and $c$,
all of unit cost, and four voters:
$v_1 : \{a\}$;
$v_2 : \{a, b\}$;
$v_3 : \{b, c\}$;
$v_4 : \{c\}$.
Then, with budget limit $1$ a winning budget might be $\{b\}$,
while with budget limit $2$ the only winning budget is $\{a, c\}$.

\mypara{$\boldsymbol{\mathcal{R}^m_{c(B_v)}}$:}
Consider a budgeting scenario $E$ with items $a$, $b$, $c$, and $d$,
with costs $2$, $3$, $3$, and $5$, respectively, 
%with $c(a) = 2$, $c(b) = 3$, $c(c) = 3$, and $c(d) = 5$,
and one voter:
$v_1 : \{a, b, c, d\}$.
Then, with budget limit $6$ the only winning budget is $\{b, c\}$,
while with budget limit $7$ the only winning budget is $\{a, d\}$.

\mypara{$\boldsymbol{\mathcal{R}^m_{|B_v|}}$:}
Consider a budgeting scenario $E$ with the same items as above
%items $a$, $b$, $c$, and $d$,
%with costs $2$, $3$, $3$, and $5$, respectively, 
%$c(a) = 2$, $c(b) = 3$, $c(c) = 3$, and $c(d) = 5$,
and five voters:
$v_1 : \{a, b, c, d\}$;
$v_2 : \{a, b, c, d\}$;
$v_3 : \{b, c, d\}$;
$v_4 : \{d\}$;
$v_5 : \{d\}$.
Then, with budget limit $6$ the only winning budget is $\{b, c\}$,
while with budget limit $7$ the only winning budget is $\{a, d\}$.
\end{proof}

Greedy and Proportional greedy rules also fail Limit Monotonicity.
  Consider a budgeting scenario with items $a$, $b$, and $c$,
  where $a$ has the largest value according to the relevant satisfaction function;
  thus, $a$ is selected in the first iteration.
  Set the cost of $a$ so that, after selecting it, the remaining budget limit is such that,
  for the original budget limit $\ell$, only $b$ can be selected,
  while for the budget limit $\ell + 1$, also $c$ can be selected.
  Set the value of $c$ to be higher than that of $b$.
  Thus, $b$ is selected in the first case and $c$ in the second case.

\begin{corollary}\label{cor:lmtwo}
  $\mathcal{R}^g_{|B_v|}$,
  $\mathcal{R}^g_{c(B_v)}$,
  $\mathcal{R}^g_{\mathbbm{1}_{|B_v| > 0}}$,
  $\mathcal{R}^p_{|B_v|}$,
  $\mathcal{R}^p_{c(B_v)}$,
  and 
  $\mathcal{R}^p_{\mathbbm{1}_{|B_v| > 0}}$,
  do not satisfy Limit Monotonicity.
\end{corollary}

A budgeting method~$\mathcal{R}$ satisfies Discount Monotonicity if
any budgeted item remains budgeted when its price decreases.  This is
a very desirable property as failing it means that people proposing
new items for the participatory budget have to think strategically
about the item's price, instead of trying to minimize it.

\begin{definition}[Discount Monotonicity]
  A budgeting method $\mathcal{R}$ satisfies Discount Monotonicity if
  for each budgeting scenario $E = (A, V, c, \ell)$ and for each
  $b \in \mathcal{R}(E)$, it holds that $b \in \mathcal{R}(E')$ for
  $E' = (A, V, c', \ell)$, where for each item $a \in A$, we have that
  $c'(a) = c(a)$ whenever $a \neq b$, and $c'(b) = c(b) - 1$.
\end{definition}

\begin{proposition}\label{proposition:discountmonotonicityone}
  $\mathcal{R}^m_{|B_v|}$,
  $\mathcal{R}^m_{\mathbbm{1}_{|B_v| > 0}}$,
  $\mathcal{R}^g_{|B_v|}$,
  $\mathcal{R}^g_{\mathbbm{1}_{|B_v| > 0}}$,
  $\mathcal{R}^p_{|B_v|}$,
  and
  $\mathcal{R}^p_{\mathbbm{1}_{|B_v| > 0}}$,
  satisfy Discount Monotonicity,
  while
  $\mathcal{R}^m_{c(B_v)}$,
  $\mathcal{R}^g_{c(B_v)}$,
  and
  $\mathcal{R}^p_{c(B_v)}$
  fail it.
\end{proposition}

\begin{proof}
Intuitively,
for $f = |B_v|$ and $f = \mathbbm{1}_{|B_v| > 0}$, decreasing the cost only increases the attractiveness of the item,
while for $f = c(B_v)$, decreasing the cost makes the item less attractive.

More formally,
let $E$ be the original election and $E'$ be the modified election,
where in $E'$ the cost of item $b$ decreases.
Then,
for $\mathcal{R}^m_{\mathbbm{1}_{|B_v| > 0}}$
and for $\mathcal{R}^m_{|B_v|}$,
assume,
counterpositively,
that there is some $B'$ which wins in $E'$ and $b \notin B'$,
but there is some $B$ which wins in $E$ and $b \in B$.
Consider the total satisfaction $TS(B)$ which $B$ achieves and the total satisfaction $TS(B')$ which $B'$ achieves;
notice first that the total satisfaction does not depend on the election (that is, the total satisfaction of some set of items achieves for $E$ is the same as it achieves for $E'$).
Notice further that $B'$ is feasible for $E$ and that, in $E'$, the total satisfaction of $B'$ is the maximum that can be achieved,
thus we have that $TS(B) \geq TS(B')$.
Now, since $B$ is feasible also in $E'$,
we conclude that $B$ shall be winning in $E'$.

For $\mathcal{R}^g_{\mathbbm{1}_{|B_v| > 0}}$,
$\mathcal{R}^g_{|B_v|}$,
$\mathcal{R}^p_{\mathbbm{1}_{|B_v| > 0}}$,
and for $\mathcal{R}^p_{|B_v|}$,
consider the iteration in which $b$ is selected for the budgeting scenario $E$,
and observe that it would be selected also for the budgeting scenario $E'$,
as its relative value does not decrease.
For $\mathcal{R}^m_{c(B_v)}$,
$\mathcal{R}^g_{c(B_v)}$,
and for $\mathcal{R}^p_{c(B_v)}$,
consider a budgeting scenario with
two items $a$, $b$,
with $c(a) = 2$ and $c(b) = 2$,
and one voter:
$v_1 : \{a, b\}$.
Let the budget limit be $2$.
Then, in this original budgeting scenario, the budget $\{b\}$ is winning,
while if the cost of $b$ decreases by one, then the only winning budget is $\{a\}$.
\end{proof}

The next two axioms regard the situation of a person proposing a new
item, provided that this new item has some internal structure and can
be presented either as a single one or as several items (e.g.,
renovation of a school can either be a single project, or several
ones, including painting the interior, painting the exterior, buying
new furniture etc.). We consider splitting and merging items.

% A budgeting method $\mathcal{R}$ satisfies Splitting Monotonicity
% if, when splitting a budgeted item, at least one of the new items is budgeted.

\begin{definition}[Splitting Monotonicity]
  A budgeting method $\mathcal{R}$ satisfies Splitting Monotonicity if
  for each budgeting scenario $E = (A, V, c, \ell)$, for each
  $a \in \mathcal{R}(E)$, and for each budgeting scenario $E'$ which
  is formed by splitting $a$ into a set of items $A'$ with
  $c(a) = c(A')$, and such that the voters which approve $a$ in $E$
  approve all items of $A'$ in $E'$ and no other voters approve items
  of $A'$, it holds that $\mathcal{R}(E') \cap A' \neq \emptyset$.
  Similarly, a budgeting method $\mathcal{R}$ satisfies Strong
  Splitting Monotonicity if
  % when splitting a budgeted item, all new
  % items are budgeted. Formally, where, for the above setting,
  it holds
  that $A' \subseteq \mathcal{R}_{E'}$.
\end{definition}

\begin{proposition}\label{proposition:splittingmonotonicityone}
  $\mathcal{R}^m_{|B_v|}$,
  $\mathcal{R}^m_{\mathbbm{1}_{|B_v| > 0}}$,
  $\mathcal{R}^g_{|B_v|}$,
  $\mathcal{R}^g_{\mathbbm{1}_{|B_v| > 0}}$,
  $\mathcal{R}^p_{|B_v|}$,
  $\mathcal{R}^p_{\mathbbm{1}_{|B_v| > 0}}$,
  $\mathcal{R}^m_{c(B_v)}$,
  and $\mathcal{R}^p_{c(B_v)}$,
  satisfy Splitting Monotonicity,
  while
  $\mathcal{R}^g_{c(B_v)}$ does not satisfy Splitting Monotonicity.
\end{proposition}

\begin{proof}
Intuitively,
$\mathcal{R}^g_{c(B_v)}$ does not satisfy Splitting Monotonicity as the new items' value is less than the original item's value. For other rules, the new items' value is at least as the original item's value, thus at least one is selected.
More formally,
let us denote the original item which is being splitted by $a$,
and the new items which $a$ is splitted into as $A = \{a_1, \ldots\}$;
we refer to $a$ is the \emph{original} item, and to items $a_i$ as the \emph{new} items.

For $\mathcal{R}^m_{\mathbbm{1}_{|B_v| > 0}}$,
$\mathcal{R}^g_{\mathbbm{1}_{|B_v| > 0}}$,
and for $\mathcal{R}^p_{\mathbbm{1}_{|B_v| > 0}}$,
observe that each of the new items covers the same number of voters as the original item.
Thus, if the original item made it to the budget, so will at least one of the new items,
and so Splitting Monotonicity follows.
Strong Splitting Monotonicity does not hold for $\mathcal{R}^m_{\mathbbm{1}_{|B_v| > 0}}$,
e.g., by considering a budgeting scenario with item $a$ and $b$,
with $a$ approved by some voters while $b$ is approved by other voters:
  While $a$ might be budgeted even if it takes the whole budget limit,
  after the splitting only one of the new items shall be budgeted, allowing more funds to be spent on $b$.
  
For $\mathcal{R}^m_{|B_v|}$,
$\mathcal{R}^g_{|B_v|}$,
and for $\mathcal{R}^p_{|B_v|}$,
observe that the value of each new item is the same as the value of the original item.
Thus, all of the new items will be budgeted, rendering these rules as Strong Splitting Monotone.

To see why the budgeting method $\mathcal{R}^m_{c(B_v)}$ satisfies Strong Splitting Monotonicity,
notice that the total satisfaction of budgeting the new items together is the same as budgeting the original item.
To see why the budgeting method $\mathcal{R}^p_{c(B_v)}$ satisfies Splitting Monotonicity,
consider the iteration in which the original item is selected,
and observe that any of the new items has the same value.
To see why the budgeting method $\mathcal{R}^g_{c(B_v)}$ does not satisfy Splitting Monotonicity,
consider, e.g., a budgeting scenario with items $a$, $b$,
with $c(a) = 3$, $c(b) = 3$.
Then, split $b$ into $b_1$, $b_2$, $b_3$, each of cost $1$,
consider one voter: $v_1 : \{a, b\}$, and budget limit $3$.
Since the value of $a$ is $3$, it would get selected in the splitted election,
while in the original election $b$ might be selected.
\end{proof}

\begin{definition}[Merging Monotonicity]
A budgeting method~$\mathcal{R}$ satisfies Merging Monotonicity
%if a set of budgeted items which are always approved together remain budgeted when merged.
%Formally,
if for each budgeting scenario $E = (V, A, c, \ell)$, and for each
$A' \subseteq \mathcal{R}(E)$ such that for each $v \in V$ 
we either have  $v \cap A' = \emptyset$ or $A' \subseteq v$, it holds that
$a \in \mathcal{R}(E')$ for
$E' = (A \setminus A' \cup \{a\}, V', c', \ell)$,
$c'(a) = \sum_{a \in A'} c(a)$, and each voter $v \in V$ for which
$A' \subseteq v$ in $E$, approves $a$ in $E'$, and no other voter
approves~$a$.
\end{definition}

\begin{proposition}\label{proposition:mergingmonotonicityone}
  $\mathcal{R}^m_{\mathbbm{1}_{|B_v| > 0}}$,
  $\mathcal{R}^m_{c(B_v)}$,
  $\mathcal{R}^g_{\mathbbm{1}_{|B_v| > 0}}$,
  $\mathcal{R}^g_{c(B_v)}$,
  and
  $\mathcal{R}^p_{c(B_v)}$,
  satisfy Merging Monotonicity,
  while
  $\mathcal{R}^p_{\mathbbm{1}_{|B_v| > 0}}$,
  $\mathcal{R}^m_{|B_v|}$,
  $\mathcal{R}^g_{|B_v|}$,
  and
  $\mathcal{R}^p_{|B_v|}$ fail it.
\end{proposition}

\begin{proof}
Intuitively,
for the satisfaction function $f(A_v, B) = |B_v|$,
the value of the original items decreases but the merged item is as expensive.
For $f(A_v, B) = c(B_v)$,
the value of the merged item equals the total value of the merged items.
For $f(A_v, B) = \mathbbm{1}_{|B_v| > 0}$ the fact that all the original items were budgeted
means that the merged item still satisfies the same voters.

More formally,
we refer to the items of $A'$ in the original budgeting scenario $E$ as the \emph{original} items,
and to the new item $a$ in the modified budgeting scenario $E'$ as the \emph{merged} item.

To see why $\mathcal{R}^m_{\mathbbm{1}_{|B_v| > 0}}$ and $\mathcal{R}^m_{c(B_v)}$ satisfy Merging Monotonicity,
assume, counterpositively, that $B$ is a winning budget of $E$ with $A' \subseteq B$,
but that no winning budget of $E'$ contains $a$; so there is some $B'$ which is winning in $E'$ but $a \notin B'$.
But then, $B'$ would achieve a strictly higher total satisfaction in $E'$,
which is equal to its total satisfaction in $E$,
and so $B'$ would be winning in $E$ as well, contradicting the assumption.

To see why $\mathcal{R}^g_{\mathbbm{1}_{|B_v| > 0}}$ and $\mathcal{R}^p_{c(B_v)}$ satisfy Merging Monotonicity,
consider the iteration in which one of the items of $A'$ is selected,
and observe that, in this iteration, $a$ would be selected,
as it gets the same increase of the total satisfaction.

To see why $\mathcal{R}^g_{c(B_v)}$ satisfies Merging Monotonicity,
consider again the iteration in which one of the items of $A'$ is selected,
and observe that $a$ would get even higher increase of the total satisfaction,
thus will be selected in this iteration as well.
To see why $\mathcal{R}^p_{\mathbbm{1}_{|B_v| > 0}}$ do not satisfy Merging Monotonicity,
observe that the proportional value increase of the merged item is smaller then that of the original items.

To see why $\mathcal{R}^m_{|B_v|}$,
$\mathcal{R}^g_{|B_v|}$,
and $\mathcal{R}^p_{|B_v|}$,
do not satisfy Merging Monotonicity,
consider a budgeting scenario with items $a$, $b$, $c$, $d$, and $e$,
with $c(a) = c(b) = c(c) = 1$ and $c(d) = c(e) = 2$.
Let the budget limit be $4$ and let there be three voters:
  $v_1 : \{a, b, c, d, e\}$;
  $v_2 : \{a, b, c, d, e\}$;
  $v_3 : \{a, b, c\}$.
Then, the winning budget is $\{a, b, c\}$,
while if we merge $a$, $b$, and $c$, into one item,
then $\{d, e\}$ would be the winning budget.
\end{proof}

\section{Experiments on Budgeting Methods}

\newcommand{\heighttt}{1.3cm}
\newcommand{\experimentoneoneline}[1]{%
$#1$ &
\includegraphics[height=\heighttt]{PLOTS/experimentone/HIST_MaxRuleNumberOfBudgetedItems_#1} &
\includegraphics[height=\heighttt]{PLOTS/experimentone/HIST_GreedyRuleNumberOfBudgetedItems_#1} &
\includegraphics[height=\heighttt]{PLOTS/experimentone/HIST_ProportionalGreedyRuleNumberOfBudgetedItems_#1} &
\includegraphics[height=\heighttt]{PLOTS/experimentone/HIST_MaxRuleOneIfSomethingIsBudgeted_#1} &
\includegraphics[height=\heighttt]{PLOTS/experimentone/HIST_GreedyRuleOneIfSomethingIsBudgeted_#1} &
\includegraphics[height=\heighttt]{PLOTS/experimentone/HIST_ProportionalGreedyRuleOneIfSomethingIsBudgeted_#1} &
\includegraphics[height=\heighttt]{PLOTS/experimentone/HIST_MaxRuleTotalBudgetedCost_#1} &
\includegraphics[height=\heighttt]{PLOTS/experimentone/HIST_GreedyRuleTotalBudgetedCost_#1} &
\includegraphics[height=\heighttt]{PLOTS/experimentone/HIST_ProportionalGreedyRuleTotalBudgetedCost_#1}
}

\begin{table*}[t]
\begin{center}
\resizebox{\textwidth}{!}{\begin{tabular}{c c c c c c c c c c}
$x$
&
$\mathcal{R}^m_{|B_v|}$ & $\mathcal{R}^g_{|B_v|}$ & $\mathcal{R}^p_{|B_v|}$ & 
$\mathcal{R}^m_{\mathbbm{1}_{|B_v| > 0}}$ & $\mathcal{R}^g_{\mathbbm{1}_{|B_v| > 0}}$ & $\mathcal{R}^p_{\mathbbm{1}_{|B_v| > 0}}$ & 
$\mathcal{R}^m_{c(B_v)}$ & $\mathcal{R}^g_{c(B_v)}$ & $\mathcal{R}^p_{c(B_v)}$
\vspace{3px}
\\
\experimentoneoneline{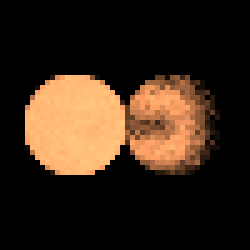} \\
\experimentoneoneline{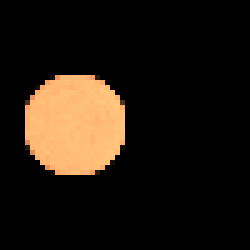} \\
\experimentoneoneline{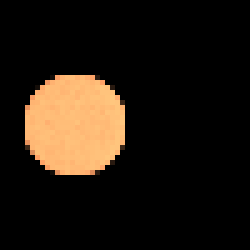} \\
\experimentoneoneline{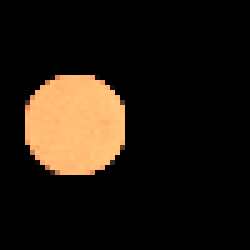} \\
\end{tabular}}
\end{center}
\caption{%
  Results of Experiment $1$.
}
\label{figure:experimentone}
\end{table*}

\newcommand{\moneyspeaksoneline}[1]{%
$#1$ &
\includegraphics[height=\heighttt]{PLOTS/experimenttwo/globalattraction#1/HIST_MaxRuleNumberOfBudgetedItems} &
\includegraphics[height=\heighttt]{PLOTS/experimenttwo/globalattraction#1/HIST_GreedyRuleNumberOfBudgetedItems} &
\includegraphics[height=\heighttt]{PLOTS/experimenttwo/globalattraction#1/HIST_ProportionalGreedyRuleNumberOfBudgetedItems} &
\includegraphics[height=\heighttt]{PLOTS/experimenttwo/globalattraction#1/HIST_MaxRuleOneIfSomethingIsBudgeted} &
\includegraphics[height=\heighttt]{PLOTS/experimenttwo/globalattraction#1/HIST_GreedyRuleOneIfSomethingIsBudgeted} &
\includegraphics[height=\heighttt]{PLOTS/experimenttwo/globalattraction#1/HIST_ProportionalGreedyRuleOneIfSomethingIsBudgeted} &
\includegraphics[height=\heighttt]{PLOTS/experimenttwo/globalattraction#1/HIST_MaxRuleTotalBudgetedCost} &
\includegraphics[height=\heighttt]{PLOTS/experimenttwo/globalattraction#1/HIST_GreedyRuleTotalBudgetedCost} &
\includegraphics[height=\heighttt]{PLOTS/experimenttwo/globalattraction#1/HIST_ProportionalGreedyRuleTotalBudgetedCost}
}

\begin{table*}[t]
\begin{center}
\resizebox{\textwidth}{!}{\begin{tabular}{c c c c c c c c c c}
$x$
&
$\mathcal{R}^m_{|B_v|}$ & $\mathcal{R}^g_{|B_v|}$ & $\mathcal{R}^p_{|B_v|}$ & 
$\mathcal{R}^m_{\mathbbm{1}_{|B_v| > 0}}$ & $\mathcal{R}^g_{\mathbbm{1}_{|B_v| > 0}}$ & $\mathcal{R}^p_{\mathbbm{1}_{|B_v| > 0}}$ & 
$\mathcal{R}^m_{c(B_v)}$ & $\mathcal{R}^g_{c(B_v)}$ & $\mathcal{R}^p_{c(B_v)}$
\vspace{3px}
\\
\moneyspeaksoneline{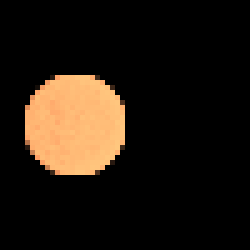} \\
\moneyspeaksoneline{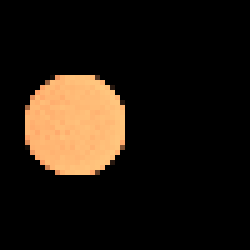} \\
\moneyspeaksoneline{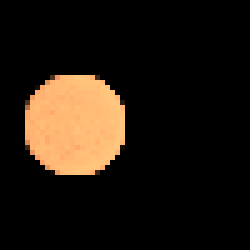} \\
\moneyspeaksoneline{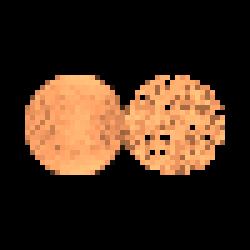} \\
\moneyspeaksoneline{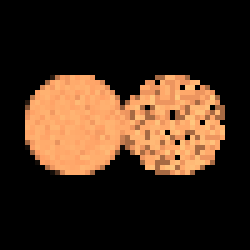} \\
\end{tabular}}
\end{center}
\caption{%
  Results of Experiment $2$.
}
\label{figure:experimenttwo}
\end{table*}

\begin{table}[t]
\begin{center}
\begin{tabular}{c}
\includegraphics[width=5.7cm]{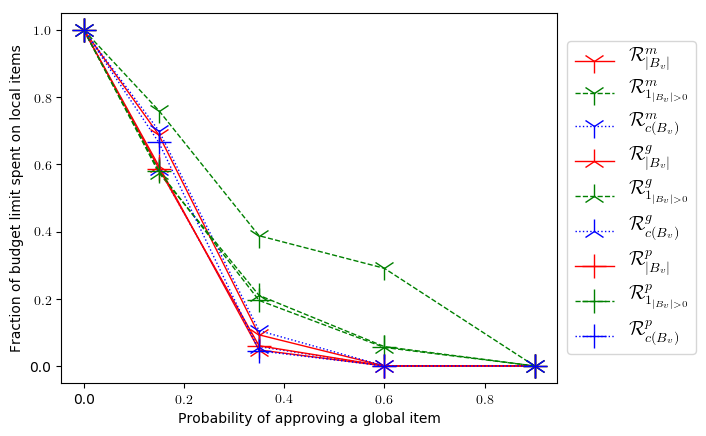} \\
\includegraphics[width=5.7cm]{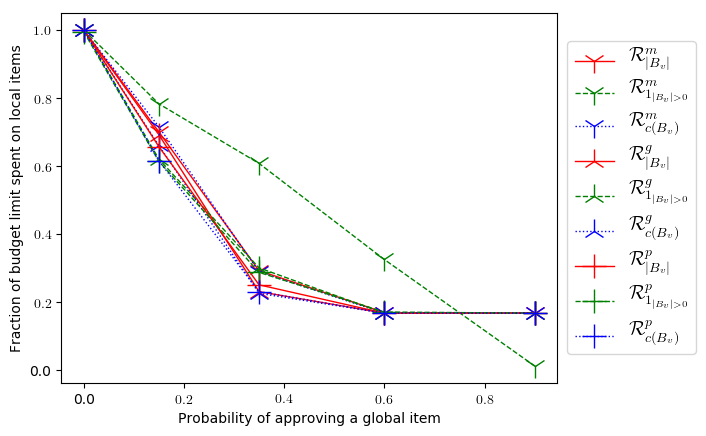} \\
\includegraphics[width=5.7cm]{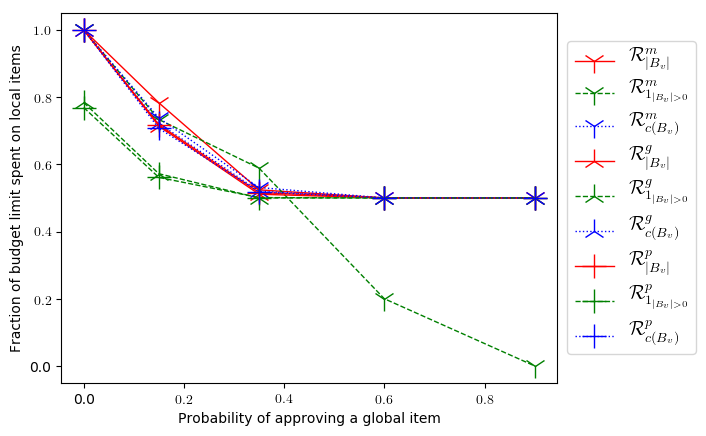} \\
\end{tabular}
\end{center}
\caption{%
  Results of Experiment $3$.
  The top, middle, and bottom figure are for $\ell = 20$, $30$, and $50$, respectively.
}
\label{figure:experimentthree}
\end{table}

In this section we report on three experiments:
  Two experiments generalize the technique of Elkind et al.~\shortcite{elkind2017multiwinner},
used quite extensively for multiwinner elections~\cite{faliszewski2018optimization,aziz2018egalitarian,faliszewski2018between},
to the setting of approval-based budgeting scenarios.
  One experiment focuses on the issue of global and local items.

In each experiment we consider the $2$-dimensional Euclidean domain,
where both voters and items correspond to ideal points on a $2$-dimensional plane;
for clarity, we set this $2$-dimensional plane to be of width $1$ and height $1$,
where $(0, 0)$, $(0, 1)$, $(1, 0)$, and $(1, 1)$, denote,
respectively, the left-bottom point, the left-top point, the right-bottom point, and the right-top point.
Concretely, each particular simulation setting consists of:
  (1) a distribution of the ideal points of the voters;
  (2) a distribution of the ideal points of the items;
  (3) a distribution of the item costs;
  (4) a budget limit;
  and
  (5) a threshold function which, based on the positions and costs, creates approval sets for the voters.

For the first two experiments,
for each particular simulation setting we generate several corresponding budgeting scenarios by sampling from these distributions, compute the winning budget in each of them, and aggregate the results into a $2$-dimensional histogram.
These histograms are formed by first partitioning the square from $(0, 0)$ to $(1, 1)$ into bins (we use $50 \times 50$ bins).
Then, we draw a pixel in each of these bins, where the more funds spent on this bin, the brighter the pixel is drawn.
Specifically, denoting the total funds used by $y$ and the funds used in a certain bin by $x$,
we normalize these values using the formula
$\frac{\arctan(\nicefrac{x}{0.0005 * y})}{\nicefrac{\pi}{2}}\ .$

\mypara{Experiment $\boldsymbol{1}$ and $\boldsymbol{2}$ (depicted in Table~\ref{figure:experimentone} and~\ref{figure:experimenttwo}, resp.).}
We describe Experiment $1$ and $2$ together, as they share a lot in common.
We have:
\begin{enumerate}

\item
voters, positioned uniformly on a disc of radius $0.3$, centered at position $(0.5, 0.5)$;
we have $50$ such voters for Experiment $1$ and $100$ for Experiment $2$;

\item
$50$ \emph{cheap} items, positioned uniformly on a disc of radius $0.2$, centered at $(0.3, 0.5)$,
and $50$ \emph{expensive} items, positioned uniformly on a disc of radius $0.2$, centered at $(0.7, 0.5)$;

\item
the cheap items cost $10$ each,
while the expensive items cost $100$ each for Experiment $2$,
and for Experiment $1$, we use a parameter $x$ for the cost of the expensive items (so Table~\ref{figure:experimentone} shows histograms for various values of $x$);

\item
the budget limit is $1000$ for Experiment $1$ and $200$ for Experiment $2$;

\item
for Experiment $1$:
  Each voter approves the $10$ items which are the closest to her;
for Experiment $2$:
  The approval sets of the voters are generated with respect to a parameter $x$, as follows:
  for each cheap item,
  we identify the $5$ voters which are the closest to it, and add the item to their approval sets;
  for each expensive item,
  we identify the $x$ voters which are the closest to it, and add the item to their approval sets.

\end{enumerate}
%
%  (1) voters, positioned uniformly on a disc of radius $0.3$, centered at position $(0.5, 0.5)$;
%  we have $50$ such voters for Experiment $1$ and $100$ for Experiment $2$;
%  (2) $50$ \emph{cheap} items, positioned uniformly on a disc of radius $0.2$, centered at $(0.3, 0.5)$,
%  and $50$ \emph{expensive} items, positioned uniformly on a disc of radius $0.2$, centered at $(0.7, 0.5)$;
%  (3) the cheap items cost $10$ each,
%  while the expensive items cost $100$ each for Experiment $2$,
%  and for Experiment $1$, we use a parameter $x$ for the cost of the expensive items (so Table~\ref{figure:experimentone} shows histograms for various values of $x$);
%  (4) the budget limit is $1000$ for Experiment $1$ and $200$ for Experiment $2$;
%  and
%  (5) for Experiment $1$:
%    The approval sets of the voters are such that each voter approves the $10$ items which are the closest to her;
%  for Experiment $2$:
%    The approval sets of the voters are generated with respect to a parameter $x$, as follows:
%    for each cheap item,
%    we identify the $5$ voters which are the closest to it, and add the item to their approval sets;
%    for each expensive item,
%    we identify the $x$ voters which are the closest to it, and add the item to their approval sets.
%
Each histogram is an aggregation of $100$ single elections.

\smallskip

Informally speaking,
the approval sets in Experiment $1$ are generated ``from the point-of-view of the voters'',
while in Experiment $2$ they are generated ``from the point-of-view of the items''.

Experiment $3$ focuses on the issue of global and local items:
  Consider, e.g., a budgeting scenario with some city-level projects and some neighborhood-level projects. Ideally, we would want some mix of city-level projects and neighborhood-level projects to be budgeted. While it is possible to achieve some mix artificially (e.g., requiring voters to select both city-level projects and neighborhood-level projects), here we are interested in finding out the natural such mix that rules in our framework achieve.

\mypara{Experiment 3 (depicted in Table~\ref{figure:experimentthree}).}
We have:
  (1) $20$ voters, positioned uniformly on the whole $1\times 1$ square; % distribution on $(0, 0)$ to $(1, 1)$;
  (2) $5$ items, termed \emph{global items}, which are also positioned uniformly on the square; % a square distribution on $(0, 0)$ to $(1, 1)$;
  and another $30$ items, termed \emph{local items}, %which are
  also positioned uniformly on the square; % a square distribution on $(0, 0)$ to $(1, 1)$;
  (3) each global item and each local item costs $5$;
  (4) we vary the budget limit between $20$ and $50$;
  (5) the approval sets of the voters are populated with respect to a parameter $p$, as follows:
  For each pair of a voter and a global item,
  we let the voter approve the global item with some probability $p$.
  For each pair of a voter and a local item,
  we let the voter approve the local item if and only if their Euclidean distance is at most $0.2$.
In Table~\ref{figure:experimentthree} each datapoint is averaged over $100$ repetitions
and we consider the average funds spent on local items as a function of the probability $p$ of approving a global item.

\subsection{Experimental Results}

Next we discuss the results of our experiments, depicted in Tables~\ref{figure:experimentone},~\ref{figure:experimenttwo},
and~\ref{figure:experimentthree}.

\newcommand{\rc}{$\mathcal{R}^m_{c(B_v)}$\xspace}
\newcommand{\rn}{$\mathcal{R}^m_{|B_v|}$\xspace}
\newcommand{\rcc}{$\mathcal{R}^m_{\mathbbm{1}_{|B_v| > 0}}$\xspace}

In Table~\ref{figure:experimentone},
and focusing on the Max rules,
it is visible that $\mathcal{R}^m_{c(B_v)}$ prefers expensive items the most,
$\mathcal{R}^m_{|B_v|}$ prefers expensive items the least,
while $\mathcal{R}^m_{\mathbbm{1}_{|B_v| > 0}}$ is somehow in between.
This can be seen specifically by noticing that
\rn ceases to select expensive items as soon as $x = 30$,
\rcc ceases to select expensive items only when $x = 190$,
while \rc keeps on selecting expensive items even when $x = 190$.
This is intuitively appealing,
as expensive items are useful for \rc and not useful for \rn;
for \rcc, while the costs of the items are not useful,
their positions are, as the rule's goal is to as many voters as possible.
The greedy rules are also somehow in between the extremes,
while, as expected, the proportional greedy rules prefer cheaper items.

This general behavior is consistent with Experiment $2$,
as can be seen in Table~\ref{figure:experimenttwo}.
Specifically,
observe that \rc switches to budgeting only expensive items as soon as $x = 10$,
while \rn switches to budgeting only expensive items only when $x = 70$;
due to the positions of the expensive items,
and due to the fact that they cover different sets of voters,
\rcc interleave cheap items with expensive items (except for the corner cases of $x = 0$ and $x = 100$).

As for locality issues,
the result depicted in Table~\ref{figure:experimentthree} for $\ell = 20$ demonstrate that \rc starts to consider global items first,
with \rn after it, and \rcc being the last to consider global items.
As $\ell$ increases,
\rc and \rn select more local items while \rcc is not affected.

One particularly interesting observation from our experiments is that the greedy rule used by
Goel~\shortcite{goel2015knapsack}, $\mathcal{R}^g_{|B_v|}$ in our language, behaves substantially differently from its Max variant, \rn. This is quite visible in the first two columns of Tables~\ref{figure:experimentone} and~\ref{figure:experimenttwo}. The Max variant gives much more attention to the cheap items than the greedy one. Thus the choice between
these two rules---already made by many users of Goel's work---may have nonnegligible consequences.

\section{Outlook}

We have defined a framework for approval-based budgeting methods and studied nine rules within it,
considered their computational and axiomatic properties, and reported on simulations to evaluate them experimentally.
Our framework, and the axiomatic properties we consider, can be used to better evaluate known budgeting methods,
as well as propose new budgeting methods, which might prove to have better theoretical guarantees and better practical behavior.
E.g., our results show that, while $\mathcal{R}^g_{|B_v|}$ is used extensively in practice (see, e.g.,~\cite{goel2015knapsack}),
it produces significantly different results than the rule which it approximates, namely $\mathcal{R}^m_{|B_v|}$.

An immediate future research direction would be to study more satisfaction functions and ways of using them,
which would correspond to more rules within our framework.
Furthermore,
defining and studying more axiomatic properties which are relevant to budgeting methods,
as well as performing more extensive experimental analysis on various budgeting methods would help in better understanding these rules.

Another avenue for future research is to seek general results for rules in the framework,
such as identifying classes of rules within the framework which satisfy certain axiomatic properties,
and better understanding which budgeting methods reside within the framework and which do not.

\bibliographystyle{plain}
\bibliography{bib}

\appendix
\clearpage

\end{document}